\RequirePackage{amsmath}
\documentclass[runningheads]{llncs}
\usepackage{lmodern}
\usepackage[utf8]{inputenc}
\usepackage{latexsym}
\usepackage[T1]{fontenc}
\usepackage{amssymb}
\usepackage{stmaryrd}
\usepackage{eulervm}
\usepackage{palatino}
\usepackage{mdwmath}
\usepackage{mdwtab}

\usepackage{graphicx}
\usepackage[cmyk]{xcolor}
\usepackage{tikz}
\usepackage{tabularx}
\usepackage{array}
\usepackage{eqparbox}
\usepackage{booktabs}

\usepackage{amssymb}
\usepackage{makecell}
\usepackage{diagbox}

\usepackage{array}
\usepackage{lineno}

\usepackage{caption}
\usepackage{url}
\usepackage{hyperref}
\usepackage{enumitem}

\newcommand{\pc}{\mathbf{P}}
\newcommand{\oc}{\mathbf{O}}

\newcommand{\bsi}{\boldsymbol{\beth}}

\usepackage{enumitem}

\title{The Representation of Meaningful Precision, and Accuracy}
\titlerunning{Precision and Rationality}
\author{\textsf{A Mani}}
\authorrunning{A Mani}
\institute{Formerly, Indian Statistical Institute, Kolkata\\
203, B. T. Road, Kolkata-700108, India\\
Email: \texttt{$a.mani.cms@gmail.com$} \\
Homepage: \url{https://www.logicamani.in}\\
Orcid: \url{https://orcid.org/0000-0002-0880-1035} }

\begin{document}

\maketitle

\abstract{The concepts of precision, and accuracy are domain and problem dependent. The simplified numeric hard and soft measures used in the fields of statistical learning, many types of machine learning, and binary or multiclass classification problems are known to be of limited use for understanding the meaningfulness of models or their relevance. Arguably, they are neither of patterns nor proofs. Further, there are no good measures or representations for analogous concepts in the cognition domain. In this research, the key issues are reflected upon, and a compositional knowledge representation approach in a minimalist general rough framework is proposed for the problem contexts. The latter is general enough to cover most application contexts, and may be applicable in the light of improved computational tools available.}

\begin{keywords}
Basic Measures of AIML, Knowledge Representation, Approximation Operators, Rational Approximations, Meaningful Precision, Accuracy, General Rough Sets, Aggregation Operations, Compositionality 
\end{keywords}

\section{Introduction}

In the context of measurement theory, \emph{accuracy} and \emph{precision} are numeric measures of observational error. The former indicates the closeness of the measurements to the ground truth or accepted value, while latter is about the similarity of repeated measurements. A measurement system is valid if it is both accurate and precise.

A data set of repeated measurements of the same quantity, is typically said to be accurate if their measure of central tendency is close to the ground value of the quantity being measured, while the set is said to be precise if their measure of variability is relatively small. These concepts may be read as generalizations of the corresponding concepts in numerical analysis, where accuracy is the closeness of the recursive approximation to the true value; while precision is again about closeness expressed in terms of significant digits. Precision does not guarantee accuracy in these cases. In statistics, \emph{bias} is the amount of inaccuracy and \emph{variability} is the amount of imprecision.
These are adapted to machine learning and AI studies with limited or poor justification, and the problem of building alternative frameworks is known to be a challenging one. 

While every concept of precision has associated structural properties, only some are justifiably quantifiable in terms of the real numbers. For example, an ancient Egyptian limestone sculpture may be said to be precisely sculpted with great skill using simple tools. The idea of precision here, is embodied by the relation of sculpted shapes of parts to some possible shapes of the parts of the real person or abstract entity depicted. The term \emph{precisely sculpted} conveys a lot more to the intended viewer such as a cue to look for certain types of possible features to compare with, and confirm that the representation is an aggregate of \emph{similarly close approximations} of some of them. \emph{Note that the mature viewer would have several skills of art appreciation, and a good sculpture would be addressing several viewership skills}.
 
In earlier work, solutions to problems of AI and ML through general rough sets are proposed by the present author \cite{am2021c,am23f,am9015,am2349} and others \cite{adc2020}. It is of natural interest to explore the mentioned problem of building reasonable theoretical frameworks with minimal intrusion and contamination \cite{am240,am501} through general rough sets.  
In this research, a general rough set based framework for the concepts of precision, and accuracy is proposed. It has the advantage of tracing the meaning of the concepts from a granular compositional perspective, and additionally decide on the validity or meaningfulness of possible numeric measures.

\section{Background}
Some essential concepts are mentioned here for the reader's convenience. For relevant basics of general rough sets, the reader is referred to \cite{amedit,am5586}. 

An \emph{information table} $\mathcal{I}$, is a tuple of the form \[\mathcal{I}\,=\, \left\langle \mathfrak{O},\, \mathbb{A},\, \{V_{a} :\, a\in \mathbb{A}\},\, \{f_{a} :\, a\in \mathbb{A}\}  \right\rangle\] with $\mathfrak{O}$, $\mathbb{A}$ and $V_{a}$ being respectively sets of \emph{Objects}, \emph{Attributes} and \emph{Values} respectively. $f_a \,:\, \mathfrak{O} \longmapsto \wp (V_{a})$ being the valuation map associated with attribute $a\in \mathbb{A}$. Values may optionally be denoted by a binary function $\nu : \mathbb{A} \times \mathfrak{O} \longmapsto \wp{(V)} $  ($V$ being a set of subsets of values) defined by for any $a\in \mathbb{A}$ and $x\in \mathfrak{O}$, $\nu(a, x) = f_a (x)$.

Relations may be derived from information tables by way of conditions of the following form: For $x,\, w\,\in\, \mathfrak{O} $ and $B\,\subseteq\, \mathbb{A} $, $(x,\,w)\,\in\, \sigma $ if and only if $(\mathbf{Q} a, b\in B)\, \Phi(\nu(a,\,x),\, \nu (b,\, w),) $ for some quantifier $\mathbf{Q}$ and formula $\Phi$. The relational system $S = \left\langle \underline{S}, \sigma \right\rangle$ (with $\underline{S} = \mathfrak{O}$) is said to be a \emph{general approximation space} ($S$ and $\underline{S}$ will be used interchangeably). In particular if $\sigma$ is an equivalence relation then $S$ is referred to as an \emph{approximation space}. In fact, every subset $K$ of attributes determines an equivalence $IND(K)$ defined by $IND(K) ab$ if and only if for every $z$ in $K$ $v(a, z) = v(b, z)$. It should be noted that objects are assumed to be defined (to the extent possible) by attributes and associated valuations. 

In classical rough sets, on the power set $\wp (S)$, lower and upper
approximations of a subset $A\in \wp (S)$ operators, apart from the usual Boolean operations, are defined as per: $A^l = \bigcup_{[x]\subseteq A} [x]$, $A^{u} = \bigcup_{[x]\cap A\neq \varnothing } [x]$, with $[x]$ being the equivalence class generated by $x\in S$. 

For the basics of partial algebras, the reader is referred to \cite{lj}. A \emph{partial algebra} $P$ is a tuple of the form \[\left\langle\underline{P},\,f_{1},\,f_{2},\,\ldots ,\, f_{n}, (r_{1},\,\ldots ,\,r_{n} )\right\rangle\] with $\underline{P}$ being a set, $f_{i}$'s being partial function symbols of arity (or place-value) $r_{i}$. The interpretation of $f_{i}$ on the set $\underline{P}$ should be denoted by $f_{i}^{\underline{P}}$; however, the superscript will be dropped in this paper as the application contexts are simple enough. If predicate symbols enter into the signature, then $P$ is termed a \emph{partial algebraic system}.   

In this paragraph the terms are not interpreted. For two terms $s,\,t$, $s\,\stackrel{\omega}{=}\,t$ shall mean, if both sides are defined then the two terms are equal (the quantification is implicit). $\stackrel{\omega}{=}$ is the same as the existence equality (sometimes written as $\stackrel{e}{=}$) in the present paper. $s\,\stackrel{\omega ^*}{=}\,t$ shall mean if either side is defined, then the other is and the two sides are equal (the quantification is implicit). Note that the latter equality can be defined in terms of the former as 
\[(s\,\stackrel{\omega}{=}\,s \, \longrightarrow \, s\,\stackrel{\omega}{=} t)\&\,(t\,\stackrel{\omega}{=}\,t \, \longrightarrow \, s\,\stackrel{\omega}{=} t) \]

A partial weak lattice is a partial algebra of the form $L =\left\langle \underline{L},\vee,\wedge  \right\rangle$ (with $\underline{L}$ being a set) that satisfies
\begin{align*}
a\wedge a =a = a\vee a. \; (a\wedge b)\wedge c \stackrel{\omega}{=} a\wedge (b\wedge c). \tag{wpl1}\\
(a\vee b)\vee c \stackrel{\omega}{=} a\vee (b\vee c) \tag{wpl2}\\
(a\wedge b)\vee a \stackrel{\omega}{=} a. \; a\wedge b \stackrel{\omega}{=} b\wedge a. \;  a\vee b \stackrel{\omega}{=} b\vee a \tag{wpl3}
\end{align*}

\subsection{Negations and Implications}
Some essential generalized implications and negations (for more details, see \cite{am9915,bsmm2022}) on a bounded quasi ordered set $L$ with top element $\top$ and bottom $\bot$ are mentioned here. 
Consider the conditions possibly satisfied by a map $n: L\longmapsto L$:
\begin{align}
n(\bot) = \top \, \&\, n(\top) = \bot \tag{N1}\\
(\forall a, b) (a\leq b \longrightarrow n(b)\leq n(a)) \tag{N2}\\
(\forall a) n(n(a)) = a \tag{N3}\\
n(a) \in \{\bot , \top \} \text{ if and only if } a= \bot \text{ or } a = \top \tag{N4}
\end{align}

$n$ is a \emph{negation} if and only if it satisfies \textsf{N1} and \textsf{N2}, while $n$ is a \emph{strong negation} if and only if it satisfies all the four conditions.

Implications satisfy a wide array of properties as they depend on the other permitted operations. Here some relevant ones are mentioned. 

A function $\bsi : L^2 \longmapsto L $ is an \emph{implication} if it satisfies (for any $a, b, c \in L$) the following:
\begin{align}
\text{If } a \leq b \text{ then } \bsi bc \leq \bsi ac \tag{First Place Antitonicity FPA}\\
\text{If } b \leq c \text{ then } \bsi ab \leq \bsi ac \tag{Second Place Monotonicity SPM}\\
\bsi \bot\bot =\top \tag{Boundary Condition 1: BC1}\\
\bsi \top \top = \top \tag{Boundary Condition 2: BC2}\\
\bsi \top \bot = \bot  \tag{Boundary Condition 3: BC3}
\end{align}

Some other properties of interest in this paper are 
\begin{align*}
\bsi ab =1 \text{ if and only if } a\leq b   \tag{Ordering Property, OP}\\
\bsi a(\bsi ab) = \bsi ab  \tag{Iterative Boolean Law, IBL}\\
\bsi aa = \top \tag{Identity Principle, IP}
\end{align*}

\subsection{Precision and Accuracy}

Extension of ROC curves to multi-class contexts from a statistical learning perspective has received some attention in recent years -- though the goal of intelligible representation is not achieved. Precision-recall plots are claimed to be superior to ROC plots (receiver operating characteristic curves) in \cite{saito15}. Other evaluation studies are \cite{juri24,chicco2021,provost00}.

\emph{Hard classification measures} proceed on the assumption that ground truth is available for verification, and are based on the number of  observations that are classified correctly or incorrectly:
\begin{itemize}
\item {\textsf{True Positive(TP)}: classification result is positive and the ground truth is positive.}
\item {\textsf{False Positive(FP)}: classification result is positive but the ground truth is negative.}
\item {\textsf{True Negative(TN)}: classification result is negative and the ground truth is negative.}
\item {\textsf{False Negative(FN)}: classification result is negative and the ground truth is positive.}
\item {\textsf{Total Positive (ToP)}: classification result is positive.}
\item {\textsf{Total Negative (ToN)}: classification result is negative.}
\end{itemize}
Denoting the corresponding cardinalities in bold, the \emph{true positive (TPR), true negative (TNR), false positive (FPR) and false negative (FNR) rates} are defined by  
\[ \text{TPR} = \dfrac{\textbf{TP}}{\textbf{ToP}}, \; \text{FPR} = \dfrac{\textbf{FP}}{\textbf{ToN}} \]
\[ \text{TNR} = \dfrac{\textbf{TN}}{\textbf{ToN}}, \; \text{FNR} = \dfrac{\textbf{FN}}{\textbf{ToP}}\]
The measures \textsf{TPR} (precision or positive prediction value ) and \textsf{TNR} (specificity) provide some information in balanced class problems. For unbalanced class problems, these would favor the majority class, and so weighting schemes are used to combine the four cardinalities. For example, Cohen's Kappa and Matthews correlation coefficient are defined respectively as the geometric and harmonic mean of the quantities 
\[\dfrac{(TP \cdot TN -FP\cdot FN)}{ToP \cdot (FP + TN)} \text{ and } \dfrac{(TP \cdot TN -FP\cdot FN)}{ToN \cdot(TP+FN)}.\] The former is a chance corrected measure of accuracy (when accuracy is measured by $\frac{\textbf{TP+TN}}{\textbf{ToP+ToN}}$).

One measure of precision is $\frac{\mathbf{TP}}{\mathbf{ToP}}$, and that of recall or sensitivity is $\frac{\mathbf{TP}}{\mathbf{TP+FN}}$. Their harmonic mean is the F-measure. 

If the class imbalance ratio $r= \frac{\textbf{ToP}}{\textbf{ToN}}$, then $\frac{\mathbf{TP}}{\mathbf{ToP}} = \frac{\mathbf{TPR}}{\mathbf{TPR} + \frac{1}{r}\mathbf{FPR}}$. Thus this measure of precision is affected by class imbalance.

\subsection{Soft Numeric Measures}

The input for ROC curves may possibly be read as probabilities or subjective probabilities, and thus the curves provide for soft classifications. In statistical learning contexts, the AUC score (value of the area under the curve integral) may be read as the probability of the event that a positive labeled observation will receive a higher probability ranking as compared to negative labeled observations. The popularity of ROCs is due to its visual appeal.

While ROC and its AUC avoid the class imbalance aspect by relying on TPR and FPR, they presume a uniform misclassification cost. Remedial measures are proposed in \cite{juri24}. The AUC computation can be deceptive for the estimation of model performance because it adds the area under curve with low sensitivity and low recall values. 
ROC and its AUC do not provide information about precision and recall. High AUC values like $0.9$ combined with low values of precision and recall like $0.3$ and $0.1$ respectively, can actually occur in practice. Precision-recall plots are more informative than ROC plots when evaluating binary classifiers on imbalanced data. In such scenarios, ROC plots may be visually deceptive with respect to
conclusions about the reliability of classification performance

In information retrieval contexts, \emph{precision} is the fraction of relevant instances among the retrieved instances, while \emph{recall} (or sensitivity) is the fraction of relevant instances that were retrieved. Balancing these is often relevant in different contexts for loose decision-making. 

In summary, in the best case scenario, \emph{statistical and machine learning perspectives of the measures may be good enough for indicating the unsuitability of the proposed model(s) for the problem}. Stochastic justifications are available for a small number of cases, and then inference is hindered.

\section{Minimalist Rough Framework}

Implementation of AIML algorithms (or computational models) correspond to abstract approximations (or partial approximations). Ideas of precision, accuracy, and related measures are relative to inter-relations between these. So general rough sets are naturally applicable. The exact number of approximation operators that need to be considered in an application context needs to be specified beforehand, and these can be many in number. For simplicity, only three pairs are used in this paper. An alternative would be to rely on large-minded reasoners \cite{am23f,am2349}. 

The framework used in this paper is necessitated by the need to have multiple $1$-place approximation and sufficient order predicates to express the reasoning about approximation operators. Additionally, $2$-place approximation operators, rational, and mereological predicates are relevant; these will be used in extended versions of the framework. In \cite{am9915}, a minimalist model for rough sets called a rough convenience lattice (\textsf{RCL}) is introduced, and it is shown that it is equivalent to RCL aggregation negation algebras, and RCL aggregation implication algebras in a perspective. However, a minimalist model such as those of a rough convenience lattice (RCL) \cite{am9915} or its well-justified abstract generalizations \textsf{CRCLANA} or \textsf{CRCLAIA} may not be optimal to build upon as the lattice order structure cannot be expected to be shared always among the many approximation operators that are desired. A bounded quasi-ordered set would be more appropriate for the purpose. The entire additional machinery of \textsf{CRCLANA} or \textsf{CRCLAIA} are actually not essential for the main concepts of this paper. It suffices to have a difference operation derived from those, and an aggregation operation. For the convenience of the reader, the theory is mostly restricted to the settings of \cite{am9915}.

Granular aspects and related types are postponed to subsequent extensions. So, in effect, three classes of partial algebraic systems are proposed for the purpose. 

\begin{definition}\label{rcon}
A partial algebraic system of the form \[{B} \, =\, \left\langle \underline{B}, l_1, u_1, l_2, u_2, l_s, u_s,   \leq, \vee,  \wedge, \bot, \top \right\rangle\] of type $(1, 1, 1, 1, 1, 1, 1, 2, 2, 2, 0, 0)$ with $(\underline{B}, \leq, \bot, \top )$ being a bounded quasi-ordered set will be said to be a \emph{precision rough convenience quasi-order} (PRCQO) if the following conditions are additionally satisfied ($\vee$ and $\wedge$ being partial weak lattice operations, and the operations $l\in \{l_1, l_2, l_s\}$ and $u \in \{u_1, u_2, u_s\}$ are generalized lower and upper approximation operators respectively):
\begin{align*}
(\forall a, b) a\vee b= b \text{ or } a\wedge b= a \longrightarrow  a\leq b \tag{wl12}\\
(\forall a, b, c) (a\vee b= c \text{ or } c\wedge b= a \longrightarrow a\leq c) \tag{wl34}\\
(\forall x) x^{ll}= x^l \leq x \leq x^u \leq x^{uu}  \tag{qlu1}\\
(\forall a, b) (a\leq b \longrightarrow a^l\leq b^l ). \; (\forall a, b) (a\leq b \longrightarrow a^u\leq b^u )    \tag{qlu-mo}\\
a^u \vee b^u \stackrel{\omega}{=} (a\vee b)^u. \; (a\wedge b)^l \stackrel{\omega}{=} a^l \wedge b^l   \tag{qlu23}\\
\top^u = \top \,\&\, \bot^{l} = \bot =\bot^u  \tag{topbot}
\end{align*}
If additionally, the quasi order is a lattice order (resp. partial order), then the system will be referred to as a \emph{precision rough convenience lattice} PRCL (resp. precision rough convenience partial order, PRCPO).
\end{definition}

In a PRCL $B$, let for any $a, b\in B$ \[a\cdot b :=a^{l_s} \wedge b^{l_s} \text{ and } a \otimes b = a^{u_s} \vee b^{u_s}.\]
These are interpreted as skeptical aggregation, and co-aggregation, and studied in detail in \cite{am9915} by the present author.
\begin{definition}\label{negrcl}
Two generalized negations are definable on a RCL, $B$ as follows:
\begin{align*}
\neg a = inf \{z : z\in B\, \& a\otimes z = \top\} \tag{addneg}\\
\sim a = sup \{z: z\in B \, \&\, a\cdot z = \bot\} \tag{mulneg}
\end{align*}
\end{definition}

These satisfy the properties specified in the next two theorems.

\begin{theorem}
\begin{align*}
(\forall a) \neg \neg a \leq a^{u_s}  \tag{WN3-N}\\
(\forall a, b) (a\leq b \longrightarrow \neg b \leq (\neg b)^{u_s} \leq (\neg a)^{u_s})   \tag{WN2-N}\\
\neg \bot \leq \top \,\&\, \neg \top = \bot  \tag{WN1-N}
\end{align*}
\end{theorem}
\begin{theorem}
\begin{align*}
(\forall a, b) (a\leq b \longrightarrow (\sim b)^{l_s} \leq (\sim a)^{l_s}\leq \sim a)   \tag{WN2-S}\\
\bot \leq \sim \top \, \& \,  \top = \sim \bot   \tag{WN3-S}
\end{align*}
\end{theorem}

The above means that $\sim$ is a weak negation. In general, a negation $n$ can be used to define binary operations $\ominus_n$ by conditiions like $a\ominus b := a^{l_s} \wedge (n(b))^{l_s}$. In what follows, $n$ will be taken as $\sim$, and it is provable that

\begin{proposition}\label{omm}
For any $a, b, c\in B$,
\begin{align}
a\ominus a = \bot   \tag{omi1}\\
a\leq b \longrightarrow a\ominus b = \bot   \tag{omi2}\\
a\leq b \longrightarrow a\ominus c \leq b\ominus c \tag{omi3}
\end{align}
\end{proposition}

\begin{definition}\label{parclaia}
An \emph{Abstract PRCQO Aggregation Implication Algebra}\\ (PRQOAI) is an algebra of the form \[{B} \, =\, \left\langle \underline{B}, \leq, \otimes, \cdot, \vee,  \wedge, l_1, u_1, l_2, u_2, l_s, u_s, \bsi_\neg, \bsi_\sim \bot, \top \right\rangle\] that satisfies for $l\in \{l_1, l_2, l_s\}$ and $u \in \{u_1, u_2, u_s\}$: 
\begin{align*}
\left\langle \underline{B}, \vee,  \wedge, l_1, u_1, l_2, u_2, l_s, u_s, \bot, \top \right\rangle \text{ is a PRCQO.}  \tag{rcl}\\
(a^{uu} = a^u \& b^{uu} = b^u \&  e^{uu} = e^u \longrightarrow a\otimes (b \otimes e)= (a\otimes b)\otimes e) \tag{wAasso1}\\
a\otimes((b\vee e)\otimes a) \stackrel{\omega}{=} ((a\vee b )\otimes c)\otimes c  \tag{wAsso2}\\
\bsi_\sim \text{ satisfies FPA, SPM, BC3, and IBL}.   \tag{imsc}\\
\bsi_\neg \text{ satisfies FPA, IP, SPM, BC1, BC2, and BC3}.   \tag{inegc}
\end{align*}
with $\cdot$ being a commutative, monoidal, order-compatible operation with identity $\bot$, and $\otimes$ additionally being a commutative, order-compatible operation with identity $\top$.
If additionally, the quasi order is a lattice order (resp. partial order), then the system will be referred to as a \textsf{PRCLAI} (resp.  \textsf{PRCPOAI}).
\end{definition}

A concept of difference between an approximation and a standard approximation can naturally be expected to access its quality relative to precision, recall, and accuracy all understood from a general rough ontological perspective (without involving any statistical assumptions). Related theory of posets with difference are studied in several papers such as \cite{hsp96} (a different adaptation to rough sets is in \cite{am2}). Some possibilities for a difference operation $-$ (corresponding to $\leq$), derived difference operation $\ominus$ that satisfies Prop \ref{omm}, a sum $\oplus$ (a commutative monoidal binary operation with identity $\bot$), an object $x$, and three approximation operators $a$, $b$, and $c$ in a \textsf{PRCLAI} are 

\begin{align*}
\amalg(x, a, b, c) = (x^{b} - x^{a})^{c}   \text{ if defined} \tag{amalg}\\
\nabla(x, a, b, c) =(x^{a} \ominus x^{b})^{c}\oplus (x^{b} \ominus x^{a})^{c}   \tag{nabla}\\
\Finv(x, l_1, l_s, u_1, u_s) = \nabla(x, l_1, l_s, u_s)\oplus \nabla(x,l_s, l_1, u_1)  \tag{finv}
\end{align*}

For a concrete definition in a \textsf{PRCLAI}, $\oplus$ can be taken as $\otimes$, and $\ominus$ is as defined earlier. Of course, these two operations and the partial difference can be defined in many other ways. All the three measures can be used to meaningfully represent precision and accuracy. 

Specific versions of the second are

\begin{align}
\nabla(x, l_1, l_s, u_s) =(x^{l_1} \ominus x^{l_s})^{u_s}\oplus (x^{l_s} \ominus x^{l_1})^{u_s}   \tag{lsymag1s}\\
\nabla(x, l_1, l_2, u_s) =(x^{l_1} \ominus x^{l_2})^{u_s}\oplus (x^{l_2} \ominus x^{l_1})^{u_s}   \tag{lsymag12s}\\
\nabla(x, u_1, u_2, u_s) = (x^{u_1} \ominus x^{u_2})^{u_s}\oplus (x^{u_2} \ominus x^{u_1})^{u_s}   \tag{usymag12s}
\end{align}

\begin{proposition}
If for all $x$, $x^{l_1}\leq x^{l_s}$ in a \textsf{PRCPOAI} then $\amalg(x, l_1, l_s, u_s) \leq x^{u_s}$ and $x^{l_s} \leq \amalg(x, l_1, l_s, l_s)$.
\end{proposition}

\begin{proof}
Since the difference operation (partial) is assumed to correspond to $\leq$, $x^{l_s} - x^{l_1} \leq x^{l_s} \leq x$. Taking the $u_s$ approximation of both sides, it follows by monotonicity that $\amalg(x, l_1, l_s, u_s) \leq x^{u_s}$. The proof of the other inequality is similar.  
\end{proof}

\emph{Meaningful precision and accuracy} can be defined in the proposed frameworks in a few nonequivalent ways for a target object $a$ as follows: 

\begin{description}
\item [Prec1]{$l_1$ is as precise as $l_2$ for $a$ if and only if $\nabla(a, l_1, l_s, u_s),$\\ $\nabla(a, l_2, l_s, u_s) \in \mathcal{S}_o$ (where $\mathcal{S}_o$ is a collection of less important sets).}
\item [Prec2]{$l_1$ is as precise as $l_s$ for $a$ if and only if $\Finv(x, l_1, l_s, u_1, u_s) \in \mathcal{S}_o$ (where $\mathcal{S}_o$ is a collection of less important sets).}
\item [EPrec3]{If the approximations are granular in the sense of the present author \cite{am240,am5586}, and additionally rational \cite{am9015}, then the definition of precision (Prec1 or Prec2) is \emph{strongly meaningful}.}
\item [EPrec4]{If the approximations are granular in the sense of the present author \cite{am5586} then the definition of precision (Prec1 or Prec2) is \emph{ meaningful} }
\item [Acc5]{$l_1$ is accurate relative to $l_s$ for $a$ if and only if $\nabla(a, l_1, l_s, u_s)\in \mathcal{S}_o$ (where $\mathcal{S}_o$ is a collection of relatively less important sets).}
\end{description}
If any of the above five properties holds for every object in a set $H$, then it holds for the whole set. In applications, objects of a type are typically relevant. $\mathcal{S}_o$ may, for example, be a set of objects that have $\bot$ as their lower approximation, and a specific object $k$ as their upper approximation.

\begin{theorem}\label{full}
The following properties hold in the above context:
\begin{align}
\nabla(x, a, b, c) =  \nabla(x, b, a, c) \tag{na1}\\
\nabla(x, a, a, c) =  \nabla(x, b, b, c)  \tag{na2}\\
\text{If for all }x,  x^{l_1}\leq x^{l_2} \text{ then }  \nabla(x, l_1, l_2, u_s) = (x^{l_2}\ominus x^{l_1})^{u_s} \leq x^{u_s}   \tag{na3}\\
\nabla(\bot, a, b, c) = \bot   \tag{na4}\\
\Finv(x, l_1, l_s, u_1, u_s) = \Finv(x, l_s, l_1, u_s, u_1)   \tag{na5}
\end{align}
\end{theorem}

\section{Extended Classical Rough Sets}

Every subset $K$ of the set of attributes $At$ of an information table $\mathcal{I}$, determines a pair $(l_K, u_K)$ of lower and upper approximations. It is necessary to compare such approximations to different extents to determine attribute reducts (subsets of attributes that provide the same quality of classification) or variants thereof. The defined set-valued measures provide a new and nonequivalent way of determining this.

Correctly speaking, a specification of preferences over sets of subsets of attributes amounts to extending classical rough sets. It is not about partial or total orders on the set of attributes (as in POAS \cite{am2} or dominance based rough sets). Suppose that $K$ is a preliminary reduct or an apparent reduct (based on possible explanations of the decision process). The associated approximations can be taken as the standard approximations $l_s$ and $u_s$. If $(l_1, u_1)$, and $(l_2, u_2)$, are two other pairs of approximations, then any object $x$ is related to the following measures:

\begin{align}
\nabla(x, l_1, l_s, u_s) =(x^{l_1} \setminus x^{l_s})^{u_s}\cup (x^{l_s} \setminus x^{l_1})^{u_s}   \tag{lsymag1sc}\\
\nabla(x, l_1, l_2, u_s) =(x^{l_1} \setminus x^{l_2})^{u_s}\cup (x^{l_2} \setminus x^{l_1})^{u_s}   \tag{lsymag12sc}\\
\nabla(x, u_1, u_s, l_s) = (x^{u_1} \setminus x^{u_s})^{l_s}\cup (x^{u_s} \setminus x^{u_1})^{l_s}   \tag{usymic1sc}\\
\end{align} 

These can be interpreted as follows (for any two approximation operators $a$ and $b$):
\begin{itemize}
\item {$\nabla(x, a, b, u_s)$ is a representation of the greatest possible difference between $x^{a}$ and $x^{b}$ relative to $u_s$}
\item {$\nabla(x, a, b, l_s)$ is a representation of the minimum definite difference between $x^{a}$ and $x^{b}$ relative to $l_s$}
\end{itemize}
Note that if $u_s$ (and therefore $l_s$) is taken as the identity operation, then $\nabla(x, a, b, u_s)$  is the symmetric difference between $x^{a}$ and $x^{b}$.

In general, not all objects, granules, and their approximations are of interest, and therefore it may suffice to define $\nabla$ for a subset $H$ of objects alone. For a specific choice of $a, b$ and $u_s$, $dom(\nabla_H) = H\times \{a\} \times\{b\}\times \{u_s\}$. Corresponding to this, the range $\mathcal{R}(\nabla_H)$ is partially ordered by set inclusion. 

\begin{theorem}
When $H$ is the set of all objects $O$, then  $\mathcal{R}(\nabla_O)$ ($\mathcal{R}(\nabla)$ for simplicity) is a lower bounded weak partial lattice. Its elements are definite objects. The same statements hold for $\mathcal{R}(\Finv)$. 
\end{theorem}
\begin{proof}
The lower or upper approximation of an empty set is the empty set. So $\emptyset \in \mathcal{R}(\nabla)$, and it serves as the lower bound of the induced set inclusion order.
The partial weak lattice operations are induced by the union and intersection operation on the power set, and are defined (for any $f, g\in \mathcal{R}(\nabla) $)
 \[f \vee g  = 
  \begin{cases}
    f\cup g & \exists h \in  \mathcal{R}(\nabla) f\cup g =h \\
    \text{undefined} & \text{otherwise }
  \end{cases}
\]
\[ f \wedge g  = 
  \begin{cases}
    f\cap g & \exists h \in  \mathcal{R}(\nabla) f\cap g =h \\
    \text{undefined} & \text{otherwise }
  \end{cases}
\]
\end{proof}

The interpretation of $\mathcal{R}(\nabla_H)$ as a measure of precision or accuracy depends on that of the standard approximations. If it is read as ground truth, then the partial algebra measures accuracy. Otherwise, it needs to read as a representation of relative accuracy. The concepts of representation of precision and accuracy in the previous section naturally carry over to the present case. 

\begin{theorem}\label{class}
The following properties hold in the above context:
\begin{align}
\nabla(x, a, b, c) =  \nabla(x, b, a, c) \tag{na1}\\
\nabla(x, a, a, c) =  \nabla(x, b, b, c)  \tag{na2}\\
\text{If for all }x,  x^{l_1}\subseteq x^{l_2} \text{ then }  \nabla(x, l_1, l_2, u_s) = (x^{l_2}\setminus x^{l_1})^{u_s} \subseteq x^{u_s}   \tag{na3+}\\
\nabla(\bot, a, b, c) = \bot = \nabla(\top, a,b ,c)  \tag{na4+}\\
\Finv(x, l_1, l_s, u_1, u_s) = \Finv(x, l_s, l_1, u_s, u_1)   \tag{na5}
\end{align}
\end{theorem}

Minimal generating subsets, valuation-related properties, and defining properties possessed by the partial algebra are of interest for characterizing the nature of possible valuations.

\subsection{Real Valued Measures}

Arbitrary cardinality based measures of rough inclusion, quality of classification, accuracy degree of approximation \cite{zpb,jzd2003} and similar measures are known not to correspond to algebraic and logical properties in general \cite{am501}. The range of $\nabla$ and $\Finv$ consist of a number of objects that may admit of valuations or grades (in the order-theoretic sense \cite{fleischer82,mkjl1995,bm1981}).
This leads to the natural question: how can algebraically compatible valuations correspond to reasoning? 

They may not be coherent in all cases, and information is bound to be lost. Assuming a finite universe, the following can be proved for classical rough sets.

\begin{theorem}
If $\mathcal{R}(\nabla)$ is a down-set bounded above by a finite number of elements of the same rank/grade, then it is a graded poset.
\end{theorem}
\begin{proof}
$\mathcal{R}(\nabla)$ is a subset of the graded poset (lattice) $\delta (O)$ of definite objects. That is, there exists a order preserving map $v: \delta(O) \longmapsto N$ such that if $b$ covers $a$, then $v(b) = v(a) +1$.

Since, by assumption it is generated by a number of definite objects with the same rank value, it must be a graded poset. 
\end{proof}

Note that the specification of $\mathcal{S}_o$ is related to that of $\mathcal{R}(\nabla)$.

\section{Discussion}

In this research, a new compositional approach to the representation of precision and accuracy from a general and a classical rough set perspective is proposed. The methodology steers clear of stochastic assumptions, and relies on granular compositionality in the axiomatic sense of the present author \cite{am5586} to make sense of the representations. In forthcoming work, the problems mentioned, and the integration with clean rough randomness \cite{am23f,am2349}, and rational approximations \cite{am9915} will be investigated for a usable framework by her. It should be stressed that this work applies as well to a few other distinct concepts of precision and accuracy of other domains.

In engineering practice, the basic goal is often to construct machine parts that adhere to the maximum permissible tolerances specified in the design. The idea of something being \emph{precisely constructed} may mean that it does not hinder some other process that involves the construction in question. Formal models of the latter need to be formalized from a distributed cognition perspective. A schema that may afford this is the following: Let $X$ be a data set, an algorithm or a process. An approximation $X^a$ of $X$ is \emph{precisely formed} relative to a subclass $\mathcal{H}$ if and only if at least one of the conditions \ref{sppf} or \ref{sopf}  holds and the mereological predicates $\pc_s$ (substantial part of) ) and $\oc_s$ (substantially overlaps)\cite{am9015}  are representable in a suitable precision space.
\begin{align*}
(\forall h\in \mathcal{H}) (\pc_s hX \longrightarrow \pc hX^a) \tag{sppf}\label{sppf}\\ 
(\forall h\in \mathcal{H}) (\oc_s hX \longrightarrow \pc hX^a) \tag{sopf}\label{sopf}  
\end{align*}

\bibliographystyle{splncs04.bst}
\bibliography{algrough23c.bib}

\section{Appendix}
An example is constructed to illustrate the main ideas. 
\begin{definition}
Let $S$ be a collection of sets (that are subsets of a $H$), and $\mathcal{G}$ be the set of neighborhoods generated by a neighborhood map $n: H\longmapsto \wp{H}$. The version of graded rough sets explored in \cite{yl96} is based on the following $k$-approximation operators and derived quantities (for any $A\in \wp(H)=S$ and positive integer $k$):
\begin{align}
A^{u_k^p} =  \{z:\, n(z)\in \mathcal{G}\, \&\, \#(n(z)\cap A) > k\}   \tag{k-upper1}\\
A^{l_k^p} = \{z:\, n(z)\in \mathcal{G}\, \&\, \#(n(z)\setminus A) \leq k\}   \tag{k-lower1}
\end{align} 
\end{definition}

If $T$ is a tolerance on the set $H$, and $\mathcal{B}$ the set of its blocks (maximal subsets $B$ of $H$ that satisfy $B^2 \subseteq T$), then the standard granular and bited approximations of a subset $A$ of $H$ are defined by \cite{am501}
\begin{align*}
X^{l} \, =\,\bigcup\{A:\,A \subseteq X\,\&\,A\in \mathcal{B}\}  \tag{lower}\\
X^{u} \, =\,\bigcup\{A:\,A\,\cap\,X\,\neq\,\emptyset \,\&\, A\in \mathcal{B}\} \tag{upper}\\
X^{u_{b}} \, =\,\bigcup\{A:\,A\,\cap\,X\,\neq\,\emptyset \,\&\, A\in \mathcal{B}\}\,\setminus (X^{c})^{l} \tag{bited-upper}
\end{align*}

Let $H \, =\,\{x_{1}, x_{2}, x_{3}, x_{4}\}$ and $T$ be a tolerance $T$ on it generated by 
\[\{(x_{1}, x_{2}),\,(x_{2},x_{3})\}.\]

Denoting the statement that the granule generated by $x_{1}$ is $(x_{1},\,x_{2})$ by $(x_{1}:x_{2})$, let the granules be the set of predecessor neighborhoods:  \[\mathcal{G}=\{(x_{1}:x_{2}),\,(x_{2}:x_{1},x_{3}),\,(x_{3}:x_{2}),\,(x_{4}:)\}\]

The different approximations (lower ($l$), upper ($u$) and bited upper ($u_b$)) are then as in Table.\ref{bitegra} below. For $k=1$, the $1$-graded approximations are in the last two columns.

\begin{table}[hbt]
\begin{tabular}{|c|c|c|c|c|c|c|}
\hline\hline
\textbf{Set} & $\mathbf{a}$ & $\mathbf{a^l}$ & $\mathbf{a^u}$  & $\mathbf{a^{u_b}}$ & $\mathbf{a^{l_1}}$ & $\mathbf{a^{u_1}}$  \\
\hline 
$A_{1}$ & $x_{1}$ & $\emptyset$ & $x_{2},\,x_{1}$  & $x_{1}$  & $\emptyset$ & $x_1, x_2$\\
\hline
$A_{2}$ & $x_{2}$ & $\emptyset$ & $x_{1}, x_{2}, x_{3}$  & $x_{1}, x_{2}, x_{3}$ & $\emptyset$ & $x_1, x_2, x_3$  \\
\hline
$A_{3}$ & $x_{3}$ & $\emptyset$ & $x_{1}, x_{2}, x_{3}$  & $x_{3}$ & $\emptyset$ & $x_1, x_2, x_3$ \\
\hline
$A_{4}$ & $x_{4}$ & $x_{4}$ & $x_{4}$ & $x_{4}$ & $\emptyset$ & $\emptyset$  \\
\hline
$A_{5}$ & $x_{1}, x_{2}$ & $x_{1}, x_{2}$ & $x_{1}, x_{2}, x_{3}$  & $x_{1}, x_{2}, x_{3}$ & $x_1, x_2$ & $x_1, x_2, x_3$  \\
\hline
$A_{6}$ & $x_{1}, x_{3}$ & $\emptyset$ & $x_{1}, x_{2}, x_{3}$  & $x_{1}, x_{2}, x_{3}$ & $\emptyset$ & $x_1, x_2, x_3$ \\
\hline
$A_{7}$ & $x_{1}, x_{4}$ & $x_{4}$ & $H$  & $x_{1}, x_{4}$ & $\emptyset$ & $x_1, x_2, x_3$\\
\hline
$A_{8}$ & $x_{2}, x_{3}$ & $x_{2}, x_{3}$ & $x_{1}, x_{2}, x_{3}$  & $x_{1}, x_{2}, x_{3}$ & $x_2, x_3$ & $x_1, x_2, x_3$ \\
\hline
$A_{9}$ & $x_{2}, x_{4}$ & $x_{4}$ & $H$  & $H$ & $\emptyset$ & $x_1, x_2, x_3$  \\
\hline
$A_{10}$ & $x_{3}, x_{4}$ & $x_{4}$ & $H$  & $x_{3}, x_{4}$ & $\emptyset$ & $x_1, x_2, x_3$ \\
\hline
$A_{11}$ & $x_{1}, x_{2}, x_{3}$ & $x_{1}, x_{2}, x_{3}$ & $x_{1}, x_{2}, x_{3}$  & $x_{1}, x_{2}, x_{3}$ & $x_1, x_2, x_3$ & $x_1, x_2, x_3$ \\
\hline
$A_{12}$ & $x_{1}, x_{2}, x_{4}$ & $x_{1}, x_{2}, x_{4}$ & $H$  & $H$ & $x_1, x_2$ & $x_1, x_2, x_3$ \\
\hline
$A_{13}$ & $x_{2}, x_{3}, x_{4}$ & $x_{2}, x_{3}, x_{4}$ & $H$  & $H$ & $ x_2, x_3$ & $x_1, x_2, x_3$ \\
\hline
$A_{14}$ & $x_{1}, x_{3}, x_{4}$ & $x_{4}$ & $H$  & $H$ & $\emptyset$ & $x_1, x_2, x_3$\\
\hline
$A_{15}$ & $H$ & $H$ & $H$  & $H$ & $x_1, x_2, x_3$ & $x_1, x_2, x_3$\\
\hline
$A_{16}$ & $\emptyset$ & $\emptyset$ & $\emptyset$ & $\emptyset$ & $\emptyset$ & $\emptyset$\\
\hline
\end{tabular}
\caption{Bited+ 1-Grade Approximations}\label{bitegra}
\end{table}

It can be checked that $\nabla(x, l_1, l, u_b)$ for the sixteen subsets in sequence are: 
\[\{\emptyset,\emptyset,\emptyset, A_4, \emptyset,\emptyset,A_4, \emptyset,A_4, A_4, \emptyset,A_4, A_4, A_4, \emptyset \}\]
If $\mathcal{S}_o$ is taken as $\{A_4, \emptyset\}$, then it is obvious that $\mathcal{R}(\nabla) = \mathcal{S}_o$. Therefore $l_1$ is an accurate approximation of $l$ relative to $\mathcal{S}_o$ and $u_b$.

Further, the sequence for $\nabla(x, l_1, l, u_1)$ is \[\{ \emptyset,\emptyset,\emptyset, \emptyset,\emptyset,\emptyset, \emptyset,\emptyset,\emptyset, \emptyset,\emptyset,\emptyset, \emptyset,\emptyset,\emptyset, \emptyset \}.\] Therefore $\mathcal{R}(\Finv) = \mathcal{R}(\nabla)$.
Thus $l$ is as precise as $l_1$.

The sequence of values for $\nabla(x, u, u_b, u_b)$ is \[\{ A_{11},\emptyset, A_{11}, \emptyset,\emptyset,\emptyset, A_{11},\emptyset,\emptyset,A_{11},\emptyset,\emptyset, \emptyset,\emptyset,\emptyset, \emptyset \}.\]  Since $A_{11}$ is most of set $H$, it can be said that $u$ is not an accurate approximation of $u_b$.

\subsection{Difference Operations}

Let $S= \langle\underline{S},\leq\rangle$ be a poset, a binary partial operation $\ominus$ is called a \emph{difference operation (or partial difference operation)} if and only if $S$ satisfies
\begin{itemize}
\item {$(a\,\leq\, b\,\leftrightarrow\,\exists{c}\,b\,\ominus\,a\,=\,c)$}
\item {$(a\,\leq\, b\,\longrightarrow\, b\,\ominus\,(b\,\ominus\, a) \,= \,a,\,b\,\ominus\, a\,\leq\, b)$}
\item {$(a\,\leq\, b\,\leq\, c\,\longleftrightarrow\, (c\,\ominus\, b)\,\leq\,(c\,\ominus\, a),(c\,\ominus\, a)\,\ominus\,(c\,\ominus\, b)\,=\,b\,\ominus\,a)$}
\end{itemize}
$(\underline{S},\leq,\ominus)$ is then a \emph{poset with difference} ($\pi\delta)$ in symbols). If $(\underline{S},\leq)$ is a bounded poset with 0 and 1, then it is a \emph{difference poset}. If $(\underline{S},\leq,\ominus)$ is a $\pi\delta$ satisfying
\[(c\leq a,b \, a\ominus c = b\ominus c \longleftrightarrow a = b)\]
then it is a \emph{poset with cancellative difference}. It is then possible to define a sum operation $\oplus$ via \[(b\ominus a) = c\longleftrightarrow(a\oplus c = b).\]
$(\underline{S},\leq,\oplus)$ satisfies 
\begin{gather*}
(a\oplus b)\stackrel{w^*}{=}(b\oplus a)\\
(a\oplus b)\oplus c\stackrel{w^*}{=}a\oplus(b\oplus c)\\
(a\oplus b = a\oplus c\longleftarrow b = c)\\
\forall{a}\exists{b} a\oplus b = a \\
((a\oplus a)\oplus b = a \longleftarrow a\oplus a = a)
\end{gather*}
where $\stackrel{w^*}{=}$ means if either side is defined then so is the other and the two are equal. $S$ is thus a \emph{cancellative partial abelian semigroup}.

Any difference poset is also a poset with cancellative difference. A poset with cancellative difference is called a \emph{ generalized difference poset}. An \emph{ideal}$K$ in a poset with difference
$(S=(\underline{S},\ominus,\leq)$ is a subset with the induced order satisfying $\forall{x}\in{S}\forall{y}\in{K}(x\leq y\longrightarrow x\in K)$. Any ideal of a difference poset is a
generalized difference poset. A key result is:

\begin{proposition}
Any generalized difference poset is realizable as the ideal of a difference poset.
\end{proposition}

In terms of the sum operation a difference poset can also be defined as a cancellative partial abelian semigroup with unity 1, such that $\forall{a}\,(a\oplus{1} = b\longrightarrow a = 0 = 1'=b)$. An
\emph{orthoalgebra} is a difference poset which has a unary complementation operation s.t 
\begin{gather*}
\forall{a}\exists{!a'},a\oplus a'= 1\\ 
(\forall{a}\,(a\oplus a = b\longrightarrow a = 0 = b = 1').
\end{gather*}

The order is recoverable in either case via $a\leq b\,\mathrm{iff}\,a\oplus c = b$ for some $c$. In an orthoalgebra $(a\oplus b)$ is defined if and only if $a\leq b'$. Difference posets are called \emph{effect algebras} when written with the $\oplus$ operation. Posets with difference can be defined without the order relation also and via $\ominus$ alone also.

\subsection{Discussion}

The compositional granular approach is superior because it
\begin{itemize}
\item {makes fewer assumptions about the data,}
\item {avoids contaminating the models in a semantic domain with unjustifiable assumptions imported from another domain,}
\item {helps to keep track of the meaning of expressions at each stage (and is therefore better suited for logic),}
\item {does not introduce unjustifiable probabilistic/subjective probabilistic assumptions and numeric measures,}
\item {the granules themselves are not specified by precision levels, and}
\item {helps to trace/fit etiologies to other methods.}
\end{itemize}

It should be mentioned that most logics/algebraic models for rough sets do not integrate/disregard numeric valuations that are typical of computational approaches.

The following example contexts illustrate aspects of these.

\textbf{Example-1}:

Suppose that we have information sheets about students of a math class across schools that provides information about their performance in different tests, activities, related subjective remarks, age, caste, economic class, gender, and summarized information from earlier information sheets. After cleaning, such data can be transformed into information tables. A socially-sensitive evaluation of their progress relative to their respective privileges, and expected standards can be based on such information.  

The attributes of such a table are mostly not comparable with others. Numeric weighting of performance in activities may be of no use because the learning benefits accumulated by a student are often not quantifiable. Valuations for the attribute would be tuples of subjective parameters. It is easy to see that converting all valuations to numeric/Boolean/integral ones cause information to be lost, and then using norms or metrics or distances to be further irretrievably lost.

Granules are best definable (for a socially sensitive classification) through Boolean case-based criteria such as \emph{If caste is $X$ and class is $Z$ and ... then the student belongs to category $L$}. Approximations constructed from such granules would be superior from the perspective of causality and meaning. Numeric precision based granules, on the other hand, naturally conflate the attributes.

The methods of the paper keep track of attributes and their subjective/non-subjective valuations. In the absence of granular compositionality, steps in the construction of approximations would be lost.

Note that, in practice, multiple pairs of approximations are typically relevant, especially when they are specified with insufficient justification.

\textbf{Example-2}:

Suppose we have a soft or hard clustering $\{(C_i, F_i) ; i=1, 2, ..., n\}$ with the cores $C_i$ being mutually disjoint, and disjoint with all frontiers $F_j$, while the frontiers may have nonempty intersection with other frontiers. The pair $(C_i, F_i)$ is read as a soft cluster for each $i$. Validation indices are known to be unreasonable, and tracing of the meaning of the soft/hard clustering from algorithms works only in a few cases.

Given general rough operators $l$ and $u$, it is always possible to compute a new collection of pairs of the form $(C_i^l, F_i^u)$. The concepts of precision and accuracy considered in the paper are natural for understanding the quality of the clustering relative to the rough perspective offered. If the operators are granular, then explanations would be representable algebraically. In my earlier papers, methods of tracing the meaning are already specified

A key step would about representing the difference between the suggested or discovered general rough clustering and the soft/hard clustering, and the relative quality of the clustering.

The paper contributes to the last part.

\textbf{Example-3}: Most other AI/ML problems (supervised or unsupervised) can be approached from a similar perspective. The paper addresses basic questions of comparison from both granular and non-granular perspectives.

\end{document}